\documentclass[table]{article}

\usepackage{microtype}
\usepackage{graphicx}
\usepackage{subfigure}
\usepackage{booktabs} % for professional tables
\usepackage{multicol}
\usepackage{multirow}
\usepackage{rotating}
\usepackage{tcolorbox}
\usepackage{bbm}
\usepackage{wrapfig}
\usepackage{bm}
\usepackage{colortbl}

\usepackage{hyperref}

\usepackage[accepted]{icml2024}

\usepackage{amsmath}
\usepackage{amssymb}
\usepackage{mathtools}
\usepackage{amsthm}

\usepackage[capitalize,noabbrev]{cleveref}

\theoremstyle{plain}
\newtheorem{theorem}{Theorem}[section]

\theoremstyle{definition}

\theoremstyle{remark}

\usepackage[textsize=tiny]{todonotes}
\icmltitlerunning{ {Reinformer: Max-Return Sequence Modeling for Offline RL}}
\begin{document}

\twocolumn[
\icmltitle{Rein\textit{for}mer: Max-Return Sequence Modeling for Offline RL}




\begin{icmlauthorlist}
\icmlauthor{Zifeng Zhuang}{zju,westlake}
\icmlauthor{Dengyun Peng}{westlake,hit}
\icmlauthor{Jinxin Liu}{westlake}
\icmlauthor{Ziqi Zhang}{westlake}
\icmlauthor{Donglin Wang}{westlake}
\end{icmlauthorlist}

\icmlaffiliation{zju}{Zhejiang University}
\icmlaffiliation{westlake}{School of Engineering, Westlake University}
\icmlaffiliation{hit}{Harbin Institute of Technology}

\icmlcorrespondingauthor{Donglin Wang}{wangdonglin@westlake.edu.cn}

\icmlkeywords{Machine Learning, ICML}

\vskip 0.3in
]



\printAffiliationsAndNotice{}  % leave blank if no need to mention equal contribution

\begin{abstract}
As a data-driven paradigm, offline reinforcement learning (RL) has been formulated as sequence modeling that conditions on the hindsight information including returns, goal or future trajectory. 
Although promising, this supervised paradigm overlooks the core objective of RL that maximizes the return.
This overlook directly leads to the lack of trajectory stitching capability that affects the sequence model learning from sub-optimal data. 
In this work, we introduce the concept of max-return sequence modeling which integrates the goal of maximizing returns into existing sequence models. 
We propose \textbf{Rein}\textbf{\textit{for}}ced Trans\textbf{\textit{for}mer} (\textbf{Rein\textit{for}mer}), indicating the sequence model is reinforced by the RL objective. 
\textbf{Rein\textit{for}mer} additionally incorporates the objective of maximizing returns in the training phase, aiming to predict the maximum future return within the distribution.
During inference, this in-distribution maximum return will guide the selection of optimal actions. 
Empirically, \textbf{Rein\textit{for}mer} is competitive with classical RL methods on the D4RL benchmark and outperforms state-of-the-art sequence model particularly in trajectory stitching ability. 
Code is public at \url{https://github.com/Dragon-Zhuang/Reinformer}.
\end{abstract}
\section{Introduction}
\label{introduction}
In classical online reinforcement learning (RL), the agent interacts with the environment to collect data and then uses that to derive the policy which maximizes the returns \citep{sutton1998introduction}.
Mainstream RL algorithms rely on fitting optimal value functions \citep{watkins1992q} or calculating policy gradients \citep{sutton1999policy}. 
Both in terms of paradigms and algorithms, reinforcement learning differs significantly from data-driven supervised learning.
To avoid expensive or even risky online interaction and reuse pre-collected datasets, offline RL is proposed.
Compared to classical RL with environment interaction, offline RL shares a paradigm more akin to supervised learning due to learning from datasets. 
This encourages researchers to explore offline algorithms from the supervised perspective.

Decision Transformer (DT) \citep{chen2021decision} maximizes the likelihood of actions conditioned on the historical trajectories (including returns), which essentially converts offline RL to a supervised sequence modeling.
Lots of subsequent works have improved DT from different perspectives, including model architecture \citep{kim2023decision}, online finetuning \citep{zheng2022online}, unsupervised pretraining \citep{xie2023future} and stitching ability \citep{wu2023elastic}.
However, these supervised paradigms seem to overlook the fundamental objective of reinforcement learning that is to maximize returns.
The only naive approach to maximize return is to manually provide an initial return that is as large as possible.
This approach is acceptable in some cases, but it becomes fatal when emphasizing trajectory stitching \citep{brandfonbrener2022does}.
A typical example, also Figure \ref{maze}, is the stitching between a fail trajectory ($R=0$: \textit{starting from the initial point but not reaching the goal}) and a successful trajectory ($R=1$: \textit{reaching the goal but not starting from the initial point}).
Ideal returns should be $0$ at first and then switch to $1$ when stitching to the successful trajectory, which is conflict with the naive max approach that manually sets $1$.

In this work, we propose the concept of max-return sequence modeling, a supervised paradigm that integrates the RL objective.
Max-return sequence modeling not only maximizes the likelihood of actions, but also predicts the maximum in-distribution returns.
Concretely, expectile regression \citep{sobotka2012geoadditive, aigner1976estimation} is adopted to make the predicted returns as close as possible to the maximum returns that are achievable under the current historical trajectory.
When performing inference, the sequence model first predicts the current maximum return and then selects the best action from the offline dataset distribution, guided by this predicted maximized return.
An implementation of max-return sequence modeling is \textbf{Rein}\textbf{\textit{for}}ced Trans\textbf{\textit{for}mer} (\textbf{Rein\textit{for}mer}), representing the sequence model reinforced by the maximum return objective.
When facing trajectory stitching in Figure \ref{maze}, \textbf{Rein\textit{for}mer} tends to predict $0$ at the initial point and predict $1$ when switching to the successful trajectory due to in-distribution max return prediction.

We exhaustively evaluate \textbf{Rein\textit{for}mer} on \texttt{Gym}, \texttt{Maze2d}, \texttt{Kitchen} and \texttt{Antmaze} datasets from the D4RL benchmark \citep{fu2020d4rl}.
\textbf{Rein\textit{for}mer} has achieved performance that is competitive with classical offline RL algorithms.
Compared with state-of-the-art sequence models, \textbf{Rein\textit{for}mer} exhibits promising improvement, especially on datasets where trajectory stitching ability is highly demanded to learn from sub-optimal data.
Our further analysis and ablation study elucidates the role of the return loss and the characteristics of the predicted maximized returns.

\section{Preliminaries}
\subsection{Offline Reinforcement Learning}
Reinforcement Learning (RL) is a framework of sequential decision.
Typically, this problem is formulated by a Markov decision process (MDP) $\mathcal{M}=\{\mathcal{S},\mathcal{A},r,p,d_0,\gamma\}$, 
with state space $\mathcal{S}$, action space $\mathcal{A}$, scalar reward function $r$, transition dynamics $p$, initial state distribution $d_0(s_0)$ and discount factor $\gamma$ \citep{sutton1998introduction}.
The objective of RL is to learn a policy, which defines a distribution over action conditioned on states $\pi\left(a_t|s_t\right)$ at timestep $t$, where $a_t \in \mathcal{A}, s_t \in \mathcal{S}$.
Given this definition, the trajectory $\tau = \left(s_0, a_0, \cdots, s_T, a_T\right)$ generated by the agent's interaction with environment $\mathcal{M}$ 
can be described as a distribution $P_{\pi}\left(\tau\right) = d_0(s_0) \prod_{t=0}^{T} \pi\left(a_t|s_t\right) p\left(s_{t+1}|s_t,a_t\right)$,
where $T$ is the length of the trajectory, and it can be infinite. The goal of RL is to find a policy $\pi$ that maximizes the expectation of the discounted cumulative return under the trajectory distribution $J\left(\pi\right) = \mathbb{E}_{\tau \sim P_{\pi}\left(\tau\right)}\left[ \sum_{t=0}^{T} \gamma^t r(s_t, a_t)\right]$. 

For \textbf{offline} reinforcement learning \citep{levine2020offline}, the interaction with the environment $\mathcal{M}$ is forbidden and only a fixed offline dataset full of transitions is provided $\mathcal{D}=\left\{ \left(s_t,a_t,r_t,s_{t+1},a_{t+1}\right)_{t=1}^{N}\right\}$ where $r_t \dot= r\left(s_t, a_t\right)$.
This setting is more challenging since the agent is unable to explore the environment and collect additional feedback.

\subsection{Sequence Modeling in Reinforcement Learning}\label{2.1}
Compared to online reinforcement learning which interacts with the environment, offline RL is more similar to the data-driven paradigm given the offline dataset $\mathcal{D}$.
As a result, offline RL \citep{chen2021decision} has been formulated as the supervised sequence modeling, different from the classical MDP formulation.
The offline dataset can be denoted as the sequence form rather than transitions $\mathcal{D} = \left\{\left(\cdots,s^{\left(n\right)}_t,a^{\left(n\right)}_t,g^{\left(n\right)}_t\cdots\right)\right\}$.
Here $g^{\left(n\right)}_t$ is the returns-to-go (or simply returns) defined as $g^{\left(n\right)}_t \dot= \sum_{t^\prime=t}^{T} r\left(s_{t^\prime}^{\left(n\right)}, a_{t^\prime}^{\left(n\right)}\right)$ that represents the sum of future rewards from current timestep $t$.
Decision Transformer (DT) \citep{chen2021decision}, following the upside-down RL \citep{srivastava2019training,schmidhuber2019reinforcement}, predicts the actions based on the previous trajectories $\tau$ concatenated with returns-to-go $g_t^{\left(n\right)}$:
\begin{align*}
    \mathcal{L}_{\mathrm{DT}} = \mathbb{E}_{t,n}\bigg[a_t^{\left(n\right)} -  \pi_{\mathrm{DT}}\left(\left<g,s,a\right>_{t-K}^{\left(n\right)};g^{\left(n\right)}_{t}, s^{\left(n\right)}_{t}\right)\bigg]^2,
\end{align*}
where $\mathbb{E}_{t, n}$ is an omission of $\mathbb{E}_{t\in\left[0,T\right], n\in\left[1,N\right]}$.
Besides, $\left<g,s,a\right>_{t-K}^{\left(n\right)}$ denotes the previous $K$ timesteps trajectory supplemented with returns-to-go $g_t^{\left(n\right)}$ and $\left<g,s,a\right>_{t-K}^{\left(n\right)} = \left(g^{\left(n\right)}_{t-K+1}, s^{\left(n\right)}_{t-K+1}, a^{\left(n\right)}_{t-K+1},\cdots,g^{\left(n\right)}_{t-1}, s^{\left(n\right)}_{t-1}, a^{\left(n\right)}_{t-1}\right)$.
In DT, the policy $\pi_{\mathrm{DT}}$ is implemented by a causal transformer, namely the decoder layers. 
For each timestep $t$, three different tokens containing returns-to-go, state and action $g^{\left(n\right)}_{t}, s^{\left(n\right)}_{t}, a^{\left(n\right)}_{t}$ are fed into the model.
And the future action $\hat{a}^{\left(n\right)}_{t}$ is predicted via autoregressive modeling.

To enable online finetuning ability, Online Decision Transformer (ODT) \citep{zheng2022online} stochastic models the action as a Gaussian distribution and trains the model by maximizing action likelihood and another max-entropy term:
\begin{align}
    \mathcal{L}_{\mathrm{ODT}} = &\mathbb{E}_{t,n}\left[-\log \pi_{\mathrm{ODT}}\left(a_t^{\left(n\right)}|\left<g,s,a\right>_{t-K}^{\left(n\right)};g^{\left(n\right)}_{t}, s^{\left(n\right)}_{t}\right)\right. \\\nonumber
    &-\left. \lambda H\left(\pi_{\mathrm{ODT}}\left(\cdot|\left<g,s,a\right>_{t-K}^{\left(n\right)};g^{\left(n\right)}_{t}, s^{\left(n\right)}_{t}\right)\right)\right],
\end{align}
where $\lambda$ is the temperature parameter \citep{haarnoja2018soft} and $\lambda$ will be adaptively updated by another temperature loss $\mathcal{L}_{\lambda}=\lambda\left(H\left(\pi_{\mathrm{ODT}}\left(\cdot|\left<g,s,a\right>_{t-K}^{\left(n\right)};g^{\left(n\right)}_{t}, s^{\left(n\right)}_{t}\right)\right) - \beta\right)$ with $\beta$ is the prefixed value \footnote{Usually, this parameter is the negative value of the action dimension $\beta = -\text{dim}\left(\mathcal{A}\right)$.}.

For the \textbf{Inference} of DT and ODT, one desired performance $\hat{g}_0$ must be specified as returns-to-go at first.
Along with the initial environment state $s_0$, the next action will be generated by the model $a_1=\pi_\mathrm{DT}\left(\hat{g}_0,s_0\right)$ or $\pi_\mathrm{ODT}\left(a_1|\hat{g}_0,s_0\right)$.
Once the action $a_1$ is executed by the environment, the next state $s_1$ and reward $r_1$ are returned.
Then the next returns-to-go should minus the returned reward $\hat{g}_1=\hat{g}_0 - r_1$.
This process is repeated until the episode terminates.
\paragraph{Drawbacks:}
These sequence models mainly focus on maximizing the action likelihood while neglecting the RL objective that maximizes the returns.
Manually setting a large initial $\hat{g}_0$ can be seen as a naive approach to maximize returns. 
Some cases are reasonable, but in the context of trajectory stitching, this method will lead to severe out-of-distribution (OOD) issues. 
It is crucial to consider the goal of maximizing returns within the framework of sequence modeling and to derive the maximum in-distribution returns during the inference phase. 

\begin{figure*}[ht]
\vspace{-6pt}
\begin{center}
\centerline{\includegraphics[width=0.89\textwidth]{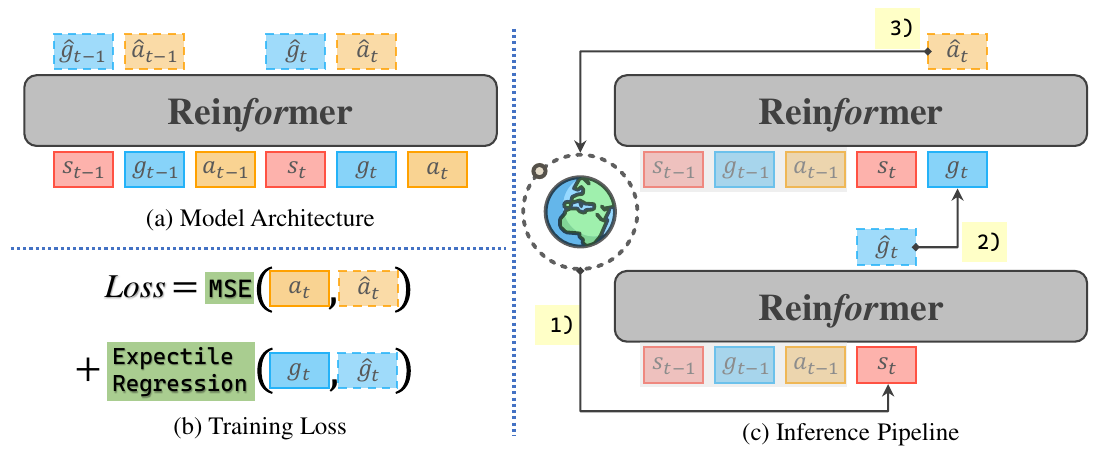}}
\vspace{-12pt}
\caption{The overview of \textbf{ReinFor}ced Trans\textbf{Former} (\textbf{Rein\textit{for}mer})
(a) Model Architecture: The returns-to-go is the second token of \textbf{Rein\textit{for}mer} inputs and the outputs contain returns and actions.
(b) Train Loss: As a max-return sequence model, \textbf{Rein\textit{for}mer} not only maximizes the action likelihood but also maximizes returns by expectile regression.
(c) Inference Pipeline: When inference, \textbf{Rein\textit{for}mer} first 1) gets state from the environment to predict the in-distribution maximum return. Then 2) predicted in-distribution max return is concatenated with state to predict the optimal action. Finally, 3) the environment executes the predicted action to return the next state.} \label{overview}
\end{center}
\vspace{-24pt}
\end{figure*}

\section{\textbf{Rein\textit{for}mer}: Reinforced Transformer}
In this section, we start with a simple maze example to illustrate why classical sequence models and the naive max-return approach are unlikely to solve the trajectory stitching problem. 
Further, we introduce the concept of max-return sequence modeling and theoretically demonstrate that this paradigm can derive maximum returns without suffering from OOD issues. 
Finally, we present the implementation details of our sequence model \textbf{ReinFor}ced Trans\textbf{Former} (\textbf{Rein\textit{for}mer}) from three aspects: model architecture, the loss function during training and the inference pipeline.

\subsection{Trajectory Stitching Example}\label{3.1}
In the offline RL literature, trajectory stitching receives lots of attention.
Ideally, the offline agent should take suboptimal trajectories that overlap and stitch them into the optimal trajectories \citep{kostrikov2021offline,liu2023ceil}.
It has been proved both theoretically \citep{brandfonbrener2022does} and empirically \citep{kumar2022offline} that the return-conditioned sequence modeling lacks stitching ability.
We utilize the following example to thoroughly describe this.

\paragraph{Example} The Figure \ref{maze} depicts a toy maze, where $s_0$ is the starting point, $s_G$ is the final goal state with
\begin{wrapfigure}[9]{r}{0.49\linewidth}
\vspace{-0.15in}
  \centering
  \includegraphics[width=0.92\linewidth]{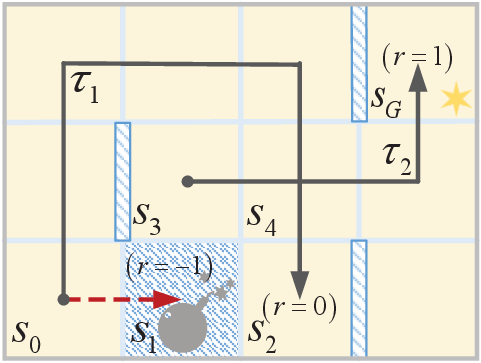} 
\vspace{-0.15in}
  \caption{A maze example for trajectory stitching analysis.}
  \label{maze}
\end{wrapfigure}
reward $r=1$, $s_1$ is a boom with $r=-1$ and other states are all $r=0$.
The offline dataset contains two trajectories one trajectory $\tau_1$ starts from the initial point $s_0$ but doesn't reach the goal while another $\tau_2$ reaches the goal $s_G$ but doesn't start from $s_0$. 
Trajectory stitching expects the offline agent can follow the first half of $\tau_1$ (from starting point $s_0$ to $s_4$) and then take the second half of $\tau_2$ (from $s_4$ to the goal $s_G$) to reach the goal.
We first explain why the basic sequence model, such as DT, might fail.

For DT, if we set initial returns-to-go as $\hat{g}_0=0$ at the starting point, the offline agent will smoothly reach the intersection state $s_4$. 
However, since returns-to-go is still zero $\hat{g}_4=0$, DT will reach the state $s_2$ rather then $s_G$.
Only when $\hat{g}_4=1$, DT is possible to follow $\tau_2$. 
But $\hat{g}_4=1$ is \textit{impossible} to achieve given $\hat{g}_0=0$.
If we apply the naive max approach and set the initial $\hat{g}_0=1$, the agent might directly walk towards the boom $s_1\left(r=-1\right)$ because $\hat{g}_0=1$ is the OOD returns-to-go for the starting point\footnote{Actually, DT or ODT use this naive approach for Antmaze dataset, called the ``delayed'' mode in code implementation.}.

If the sequence model is endowed with capability to maximize the returns like RL, Let's see what might happen.
At the starting point $s_0$, only $\tau_1$ is contained in dataset so the model will predict $\hat{g}_0=0$.
When offline agent comes to the intersection $s_4$, the latter segments of both trajectories are available.
If the sequence model is able to maximize return, then $\tau_2$ is more likely to be selected since the return $R=1$ is larger.
This inspires us to bring the capability of maximizing returns back into sequence modeling.
\subsection{Max-Return Sequence Modeling}
The key objective of RL is to obtain the optimal action for the current state through maximizing the returns. 
We aim to equip supervised sequence modeling with additional maximizing return objective.
And during inference, the sequence model can select optimal action conditioned on the in-distribution maximized returns.
We introduce the expectile regression as returns-to-go loss to achieve this.

Expectile regression \citep{newey1987asymmetric} is well studied in applied statistics and econometrics and has been introduced into offline RL recently \citep{kostrikov2021offline, wu2023elastic}.
Specifically, the returns-to-go loss based on the expectile regression is as follows:
\begin{align}
    \mathcal{L}^{m}_{\text{g}}=\mathbb{E}_{t, n}\left[\left|m-\mathbbm{1} \left( \Delta g < 0\right) \right|\Delta g^2\right],
\end{align}
here $\Delta g = g_t^{(n)} - \textcolor{blue}{\hat{g}_t^{(n)}}$ and $\textcolor{blue}{\hat{g}_t^{(n)}} = \pi_{\theta}\left(\left<s,g,a\right>^{\left(n\right)}_{t-K};s^{\left(n\right)}_t\right)$.
It is noteworthy that the relative positions of these tokens are different from DT, where returns-to-go $g_t^{\left(n\right)}$ is placed after the state token $s_t^{\left(n\right)}$ and before the action $a_t^{\left(n\right)}$.
Here $m \in \left(0,1\right)$ is the hyperparameter of expectile regression. 
When $m=0.5$, expectile regression degenerates into standard regression, also MSE loss.
\begin{wrapfigure}[10]{l}{0.49\linewidth}
\vspace{-0.15in}
  \centering
  \includegraphics[width=0.95\linewidth]{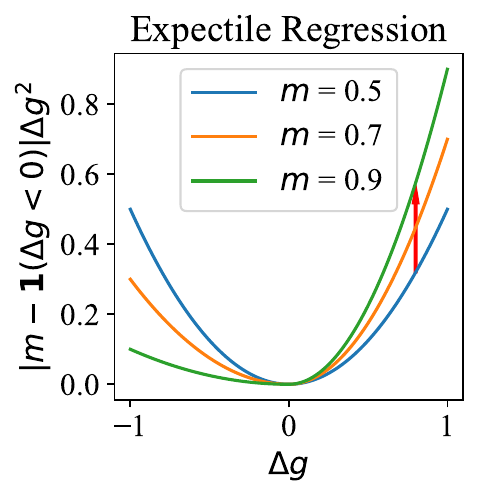} 
\vspace{-0.25in}
  \caption{illustration of weight.}
  \label{ER}
\end{wrapfigure}
But when $m>0.5$, this asymmetric loss will give more weights to the $g_t^{\left(n\right)}$ larger than $\textcolor{blue}{\hat{g}_t^{\left(n\right)}}$, which aligns with the asymmetric curves in Figure \ref{ER}.
Besides, The red arrow shows the weight increases as the $m$ becomes larger. 
In other words, the predicted returns-to-go $\textcolor{blue}{\hat{g}_t^{\left(n\right)}}$ will approach larger $g_t^{\left(n\right)}$.

To unveil what the returns-to-go loss function has learned and offer a formal elucidation of its role, we introduce the following theorem along with its proof:
\begin{theorem}\label{thm}
We first define $\mathbf{S}^{\left(n\right)}_t \dot= \left[\left<s,g,a\right>^{\left(n\right)}_{t-K};s^{\left(n\right)}_t\right]$.
For $m\in\left(0,1\right)$, denote $\mathbf{g}^m\left(\mathbf{S}^{\left(n\right)}_t\right) = \pi_{\theta}^*\left(\mathbf{S}^{\left(n\right)}_t\right)$, where $\pi_{\theta}^* = 
\arg \min \mathcal{L}_{\text{g}}^m$, then we have 
\begin{align*}
    \lim_{m\rightarrow 1} \mathbf{g}^m\left(\mathbf{S}^{\left(n\right)}_t\right) = g_{\text{max}},
\end{align*}
where $g_{\text{max}} = \max_{\mathbf{a} \sim \mathcal{D}} g\left(\mathbf{S}^{\left(n\right)}_t, \mathbf{a}\right)$ denotes the maximum returns-to-go with actions from offline dataset.
\end{theorem}
\begin{proof}
First, according to the monotonicity of expectile regression, we have $\mathbf{g}^{m_1} \leq \mathbf{g}^{m_2}$ when $0<m_1<m_2<1$.

Secondly, $\forall m \in \left(0,1\right), \mathbf{g}^{m} \leq g_{\text{max}}$ and we can use contradiction to prove it.
Assume one $m_3$ satisfies $\mathbf{g}^{m_3} > g_{\text{max}}$, then all the returns-to-go from offline dataset will $g_t^{(n)} < \mathbf{g}^{m_3}$.
The returns-to-go loss can be simplified given the same weight $1-m_3$:
\begin{align*}
    \mathcal{L}^{m_3}_{\textbf{g}}&=\mathbb{E}_{t, n}\left[\left(1-m_3 \right)\left(g_t^{(n)} - \mathbf{g}^{m_3}\right)^2 \right] \\
    &> \mathbb{E}_{t, n}\left[\left(1-m_3 \right)\left(g_t^{(n)} - \max_{t,n} [g_t^{(n)}]\right)^2 \right].
\end{align*}
This inequality holds because $g_t^{(n)} \leq \max_{t,n} [g_t^{(n)}] < \mathbf{g}^{m_3}$.
But this inequality is conflict with the fact that $\mathbf{g}^{m_3}$ is obtained by minimizing the returns-to-go loss.
Therefore, the assumption is not valid and $\mathbf{g}^{m} \leq g_{\text{max}}$ is true.

Finally, this limit follows from the properties of bounded monotonically non-decreasing functions thus affirming the validity of the theorem.
\end{proof}

In one word, Theorem \ref{thm} indicates the loss $\mathcal{L}_{\text{g}}^m$ will make the model predict the maximum returns-to-go when $m\rightarrow 1$, which is similar to the maximizing returns objective in RL.
When inference, the model will generate the action conditioned on this predicted maximum returns-to-go.
Furthermore, the second step in our proof indicates that this predicted returns do not suffer from OOD issues.

\subsection{Implementation Details}
Now, we will focus on the specific implementation of \textbf{Rein\textit{for}mer}, describing the model input and output, training, and inference procedures.
Figure \ref{overview} is an overview.
\paragraph{Model Architecture} 
To accommodate the max-return sequence modeling paradigm, which predicts the maximum return and utilizes it as a condition to guide the generation of optimal actions, we have positioned returns between state and action. 
In detail, the input token sequence of \textbf{Rein\textit{for}mer} and corresponding output tokens are summarized as follows:
\begin{align*}
    \textbf{Input: \ }&\left<\ \cdots,s^{\left(n\right)}_t, g^{\left(n\right)}_t, a^{\left(n\right)}_t\right> \\
    \textbf{Output: \ }&\quad \quad \left<\ \hat{g}^{\left(n\right)}_t, \hat{a}^{\left(n\right)}_t,\Box\ \right>
\end{align*}
here $g^{\left(n\right)}_t \dot= \sum_{t^\prime=t}^{T} r\left(s_{t^\prime}^{\left(n\right)}, a_{t^\prime}^{\left(n\right)}\right)$ is the returns-to-go and we will abbreviated it as \textit{returns} in the absence of ambiguity.
When predicting the $\hat{g}_t^{\left(n\right)}$, the model takes the current state $s_t^{\left(n\right)}$ and previous $K$ timesteps tokens $\left<s,g,a\right>_{t-K}^{\left(n\right)}$ into consideration.
For the sake of simplicity, $\mathbf{S}^{\left(n\right)}_{t-K}$ denotes the input $\left[\left<s,g,a\right>^{\left(n\right)}_{t-K};s^{\left(n\right)}_t\right]$.
While the action prediction $\hat{a}_t$ is based on $\left(\mathbf{S}^{\left(n\right)}_{t-K}, \mathbf{G}^{\left(n\right)}_{t-K}\right) = \left[\left<s,g,a\right>^{\left(n\right)}_{t-K};s^{\left(n\right)}_t,g^{\left(n\right)}_t\right]$.
The $\Box$ represents this predicted token neither participates in training nor inference so we ignore it.

At the timestep $t$, different tokens are embedded by different linear layers and fed into the transformers \citep{vaswani2017attention} together. 
The output returns-to-go $\hat{g}^{\left(n\right)}_t$ is processed by a linear layer.
For action $\hat{a}^{\left(n\right)}_t$, we adopt the Gaussian distribution and use the mean of this distribution for inference.

\paragraph{Training Loss} Since the model predicts two parts, $\hat{g}_t$ and $\hat{a}_t$, the loss function is composed of returns-to-go loss and action loss.
For the action loss, we adopt the loss function of ODT and simultaneously adjust the order of tokens:

\begin{align}\label{action_loss}
    \mathcal{L}_{\text{a}} &= \mathbb{E}_{t,n}\left[-\log \pi_{\theta}\left(a_t^{\left(n\right)}|\mathbf{S}^{\left(n\right)}_{t-K}, \mathbf{G}^{\left(n\right)}_{t-K}\right) \right. \nonumber \\
    &\quad\quad -\left. \lambda H\left(\pi_{\theta}\left(\cdot|\mathbf{S}^{\left(n\right)}_{t-K}, \mathbf{G}^{\left(n\right)}_{t-K}\right)\right)\right]
\end{align}

The returns-to-go loss is the expectile regression with the parameter $m$:
\begin{align}\label{return_loss}
    \mathcal{L}^{m}_{\text{g}}=\mathbb{E}_{t, n}\left[\left|m-\mathbbm{1} \left( \Delta g < 0\right) \right|\Delta g^2\right]&,\\ \nonumber
    \text{with \ } \Delta g = g_t^{(n)} - &\pi_{\theta}\left(\mathbf{S}^{\left(n\right)}_{t-K}\right).
\end{align}

Two loss functions have the same weight so the total loss is $\mathcal{L}_{\text{a}} + \mathcal{L}_{\text{g}}$.
Algorithm \ref{training} summarizes the training process.

\begin{algorithm}[t]
\caption{Training}
\begin{algorithmic}[1]\label{training}
\STATE {\bfseries Input:} offline dataset $\mathcal{D}$, sequence model $\pi_{\theta}$
\FOR{ sample $ \left<\ \cdots,s_{t}, g_{t}, a_{t}\ \right>$ from $\mathcal{D}$}
    \STATE $\hat{g}_t, \hat{a}_t=\pi_{\theta}\left(\cdots,s_{t}, g_{t}, a_{t}\right)$
    \STATE Calculate action loss $\mathcal{L}_{\text{a}}$ by Equation (\ref{action_loss})
    \STATE Calculate returns-to-go loss $\mathcal{L}_{\text{g}}^m$ by Equation (\ref{return_loss})
    \STATE Take gradient descent step on $\nabla_{\theta}\left(\mathcal{L}_{\text{a}} + \mathcal{L}_{\text{g}}^m\right)$
\ENDFOR
\end{algorithmic}
\end{algorithm}

\paragraph{Inference Pipeline}
For each timestep $t$, the action is the last token, which means the predicted action is affected by state from the environment and the returns-to-go.
The returns of the trajectories output by the sequence model exhibit a positive correlation with the initial conditioned returns-to-go \citep{chen2021decision,zheng2022online}. 
That is, within a certain range, higher initial returns-to-go typically lead to better actions.
In classical Q-learning \citep{mnih2015human}, the optimal value function $Q^*$ can derive the optimal action $a^*$ given the current state. 
In the context of sequence modeling, we also assume that the maximum returns-to-go are required to output the optimal actions. 
The inference pipeline of the \textbf{Rein\textit{for}mer} is summarized as follows:
\begin{align}
    \overset{\text{\textcolor{blue}{Env}}}{\longmapsto} s_0 \xrightarrow{\pi_{\theta}} g_0 \xrightarrow{\pi_{\theta}} a_0 \xrightarrow{\text{\textcolor{blue}{Env}}} s_1 \xrightarrow{\pi_{\theta}} g_1 \xrightarrow{\pi_{\theta}} a_1 \rightarrow \cdots
\end{align}
Specially, the environment initializes the state $s_0$ and then the sequence model $\pi_{\theta}$ predicts the maximum returns-to-go $g_0$ given current state $s_0$.
Concatenating $g_0$ with $s_0$, $\pi_{\theta}$ can output the optimal action $a_0$.
Then the environment transitions to the next state $s_1$ and the reward $r_1$.
It should be noted that this reward $r_1$ will \textbf{not} participate in the inference.
Repeat the above steps until the trajectory comes to an end.
This pipeline has been summarized in Algorithm \ref{inference}.

\begin{algorithm}[t]
    \caption{Inference}
\begin{algorithmic}[1]\label{inference}
    \STATE {\bfseries Input:} sequence model $\pi_{\theta}$, environment $\text{Env}$
    \STATE $s_0 = \text{Env}.reset(\ )$ and $t=0$
    \REPEAT
    \STATE Predict maximum returns $\textcolor{blue}{\hat{g}_t}=\pi_{\theta}\left(\cdots,s_{t}, \Box, \Box\ \right)$
    \STATE Predict optimal action $\hat{a}_t=\pi_{\theta}\left(\cdots,s_{t}, \textcolor{blue}{\hat{g}_{t}}, \Box \right)$
    \STATE $s_{t+1},r_t= \text{Env}.step(\hat{a}_t)$ and $t=t+1$
    \UNTIL{done}
\end{algorithmic}
\end{algorithm}

\subsection{Comparison with EDT} Elastic Decision Transformer (EDT) \citep{wu2023elastic} explicitly considers trajectory stitching, so we briefly introduce EDT and then compare it with our \textbf{Rein\textit{for}mer}.
EDT also introduces the same returns-to-go loss (\ref{return_loss}). 
During the inference phase, EDT uses this loss to estimate the maximum returns-to-go $\hat{g}_{\text{max}}(K)$ achievable under different historical lengths $K$. 
EDT then selects the maximum $\hat{g}^*_{\text{max}} = \max_K \hat{g}_{\text{max}}(K)$ and finally determine the action based on the expert action inference \citep{lee2022multi} along with the EDT model.
The inference of EDT dynamically adjusts historical trajectories, which preserves longer historical lengths when previous trajectories are optimal and shortens them when they are suboptimal.

Similarly, EDT requires initial returns as input. 
Therefore, we also categorize EDT as a kind of basic sequence model that does not explicitly consider maximizing return. 
Although EDT is equipped with the same returns-to-go loss, it still relies on the naive max-return. 
For the \texttt{Antmaze} dataset that requires trajectory stitching, EDT does not provide experimental results. 
Our reproduced results indicate that EDT perform poorly on \texttt{Antmaze}.
% In contrast, \textbf{Rein\textit{for}mer} is inspired by the notion in reinforcement learning that maximizing returns can lead to optimal actions. 
% It directly allows the model to output the optimal action based on the maximum returns-to-go. 
% Additionally, EDT requires the manual specification of initial returns-to-go, whereas our method does not.
\begin{table*}[htbp]
  \centering
  \caption{The normalized last score on D4RL \texttt{Gym-v2}, \texttt{Maze2d-v1} and \texttt{Kitchen-v0} dataset. We report the mean and standard deviation of normalized score for five seeds. For each seed, the stats is calculated by 10 evaluation trajectories for \texttt{Gym} while 100 for \texttt{Maze2d} and \texttt{Kitchen}. The best result is \textbf{bold} and the \textcolor{blue}{blue} result means the best result among sequence modeling.}
  \small
    \begin{tabular}{c|l|rrr|rrrrr@{\hspace{-0.1pt}}l}
    \toprule
    \multicolumn{2}{c|}{\multirow{2}[4]{*}{Dataset}} & \multicolumn{3}{c|}{Reinforcement Learning} & \multicolumn{6}{c}{Sequence Modeling} \\
\cmidrule{3-11}    \multicolumn{2}{c|}{} & \multicolumn{1}{c}{BC} & \multicolumn{1}{c}{CQL} & \multicolumn{1}{c|}{IQL} & \multicolumn{1}{c}{DT} & \multicolumn{1}{c}{ODT} & \multicolumn{1}{c}{EDT} & \multicolumn{1}{c}{QDT} & \multicolumn{2}{c}{\textbf{Rein\textit{for}mer}} \\
    \midrule
    \multirow{10}[4]{*}{\begin{sideways}\texttt{Gym}\end{sideways}} & \texttt{halfcheetah-medium} & 42.6  & 44.0  & \textbf{47.4}  & 42.6  & 42.7  & 42.5  & 39.3  & \textcolor{blue}{42.94}&\textcolor{blue}{$\pm$0.39} \\
          & \texttt{halfcheetah-medium-replay} & 36.6  & \textbf{45.5}  & 44.2  & 36.6  & \textcolor{blue}{40.0}  & 37.8  &   35.6 & 39.01 & $\pm$0.91 \\
          & \texttt{halfcheetah-medium-expert} & 55.2  & 91.6  & 86.7  & 86.8  &       &    91.2   &   & \textcolor{blue}{\textbf{92.04}} &\textcolor{blue}{\textbf{$\pm$0.32 }}\\
          & \texttt{hopper-medium} & 52.9  & 58.5  & 66.3  & 67.6  & 67.0  & 63.5  & 66.5  & \textcolor{blue}{\textbf{81.60}}&\textcolor{blue}{\textbf{$\pm$3.32 }}\\
          & \texttt{hopper-medium-replay} & 18.1  & \textbf{95.0}  & 94.7  & 82.7  & 86.6  & \textcolor{blue}{89.0}  & 52.1  & 83.31 & $\pm$3.47 \\
          & \texttt{hopper-medium-expert} & 52.5  & 105.4  & 91.5  & 107.6  &       & 107.8  &   & \textcolor{blue}{\textbf{107.82}} & \textcolor{blue}{\textbf{$\pm$2.14}} \\
          & \texttt{walker2d-medium} & 75.3  & 72.5  & 78.3  & 74.0  & 72.2  & 72.8  & 67.1  &\textcolor{blue}{\textbf{ 80.52}} & \textcolor{blue}{\textbf{$\pm$2.74} }\\
          & \texttt{walker2d-medium-replay} & 26.0  & \textbf{77.2}  & 73.9  & 66.6  & 68.9  & \textcolor{blue}{74.8}  & 58.2  & 72.89 & $\pm$5.06 \\
          & \texttt{walker2d-medium-expert} & 107.5  & 108.8  & \textbf{109.6}  & 108.1  &       & 107.9  &   & \textcolor{blue}{109.35} & \textcolor{blue}{$\pm$0.32} \\
\cmidrule{2-11}          & \textit{Total} & \textit{466.7}  & \textit{698.5}  & \textit{692.6}  & \textit{672.6}  &       & \textit{687.3}  &   &\textit{\textcolor{blue}{\textbf{709.46}}}  \\
    \midrule
    \multirow{4}[4]{*}{\begin{sideways}\texttt{Maze2d}\end{sideways}} & \texttt{maze2d-umaze} & 0.4   & -8.9  & 42.1  & 18.1  &       &    35.8   &    \textcolor{blue}{\textbf{57.3}}   & 57.15 & $\pm$4.27 \\
          & \texttt{maze2d-medium} & 0.8   & \textbf{86.1}  & 34.9  & 31.7  &       & 18.3  &    13.3   & \textcolor{blue}{85.62} & \textcolor{blue}{$\pm$30.89} \\
          & \texttt{maze2d-large} & 2.3   & 23.8  & \textbf{61.7}  & 35.7  &       & 26.8  &   31.0    & \textcolor{blue}{47.35} &\textcolor{blue}{ $\pm$6.89 }\\
\cmidrule{2-11}          & \textit{Total} & \textit{3.5}   & \textit{101.0}  & \textit{138.7}  & \textit{85.5}  &       & \textit{85.9}  &    \textit{101.6}   & \textit{\textcolor{blue}{\textbf{190.12}}} &  \\
    \midrule
    \multirow{4}[4]{*}{\begin{sideways}\texttt{Kitchen}\end{sideways}} & \texttt{kitchen-complete} & \textbf{65.0}  & 43.8  & 62.5  &       &       & 40.8  &       & \textcolor{blue}{59.85}& \textcolor{blue}{$\pm$0.52} \\
          & \texttt{kitchen-partial} & 38.0  & 49.8  & 46.3  &       &       & 10.0  &  & \textcolor{blue}{\textbf{73.10}} & \textcolor{blue}{\textbf{$\pm$2.01}} \\
          & \texttt{kitchen-mixed} & 51.5  & 51.0  & 51.0  &       &       & 9.8   & & \textcolor{blue}{\textbf{65.50}} & \textcolor{blue}{\textbf{$\pm$6.22}} \\
\cmidrule{2-11}          & \textit{Total} & \textit{154.5}  & \textit{144.6}  & \textit{159.8}  &       &       & \textit{60.5}  &       & \textit{\textcolor{blue}{\textbf{198.45}}} & \\
    \bottomrule
    \end{tabular}%
  \label{tab1}%
  \vspace{-14pt}
\end{table*}%
\section{Related Work}
\paragraph{Offline Reinforcement Learning} As the cradle of supervised sequence modeling, offline RL \citep{levine2020offline} breaks free from the traditional paradigm of online interaction. 
Classic off-policy actor-critic algorithms \citep{degris2012off, konda1999actor} encounter out-of-distribution (OOD) issues, where the value function tends to overestimate the OOD state-action pairs \citep{fujimoto2019off,liu2024beyond}. 
Mainstream offline algorithms can be categorized into two classes.
One is policy constraint, which constrains the learned policy stay close to the behavior policy based on different ``distance'' such as batch constrained \citep{fujimoto2019off}, KL divergence \citep{wu2019behavior,liu2022dara}, MMD distance \citep{kumar2019stabilizing} and MSE constraint \citep{fujimoto2021minimalist}.
The other is value regularization, which regularizes the value function to assign low values on OOD state-action pairs \citep{liu2024design,kostrikov2021offline1,bai2022pessimistic}.
Some algorithms take a different approach, understanding and solving offline RL problems from the perspective of on-policy learning. 
R-BVE \citep{gulcehre2021regularized} and Onestep RL \citep{brandfonbrener2021offline} both transform off-policy style offline algorithms (such as CRR \citep{wang2020critic}, BCQ \citep{fujimoto2019off}, BRAC \citep{wu2019behavior}) into on-policy style.
Besides, BPPO \citep{zhuang2023behavior} finds the online algorithm PPO \citep{schulman2017proximal} can directly solve the offline RL due to its inherent conservatism.
In contrast, Decision Transformer (DT) \citep{chen2021decision} directly maximizes the action likelihood, which opens up a new paradigm called supervised sequence modeling. 

\paragraph{Sequence Modeling in Reinforcement Learning} 
Before Decision Transformer (DT), upside-down reinforcement learning \citep{srivastava2019training,schmidhuber2019reinforcement} has already begun exploring RL solutions using supervised learning methods.
DT \citep{chen2021decision} incorporates return as part of the sequence to predict the optimal action.
This paradigm breaks away from the classic RL paradigm, circumventing OOD problems directly. 
Inspired by DT, RL is investigated from a supervised learning perspective, including network architecture \citep{kim2023decision, david2022decision}, unsupervised pretraining \citep{xie2023future} and large capacity model \citep{lee2022multi}.
While other works improve DT using RL components such as online finetuning \citep{zheng2022online}, trajectory stitching \citep{wu2023elastic} and dynamics programming \citep{yamagata2023q}.
However, no sequence model incorporates the RL objective that maximizes returns to enhance the model \citep{liu2024didi}.

\section{Experiments}
We conduct extensive experiments on \texttt{Gym}, \texttt{Maze2d}, \texttt{Kitchen} and \texttt{Antmaze} datasets from D4RL benchmark \citep{fu2020d4rl}, aiming to answer the following questions:
\begin{itemize}
    \item As a max-return sequence modeling method, how does \textbf{Rein\textit{for}mer} compare with other state-of-the-art sequence models that do not explicitly consider maximizing returns? Can \textbf{Rein\textit{for}mer} narrow the performance gap or even surpass classical offline RL algorithms? 
    \item How does the predicted maximized returns-to-go and the parameter $m$ of returns loss affect the \textbf{Rein\textit{for}mer} performance? Additionally, what is the characteristic of the predicted maximized returns-to-go?
\end{itemize}
\subsection{Results on D4RL Benchmark}
We conduct a comprehensive evaluation of \textbf{Rein\textit{for}mer}, covering not only \texttt{Gym} dataset where performance approaches saturation but also more challenging datasets like \texttt{Maze2d}, \texttt{Kitchen} and \texttt{Antmaze}.
Among these, \texttt{Antmaze} is characterized by sparse rewards and strongly emphasizes the stitching ability.

\paragraph{Baselines}
We compare our method with both representative reinforcement learning and state-or-the-art sequence modeling methods.
Some results are reproduced using the official code and please refer to Appendix \ref{a1} for detail.
\begin{itemize}
    \item \textbf{Reinforcement learning} includes Behavior Cloning (BC) \citep{pomerleau1988alvinn}, Conservative Q-Learning (CQL) \citep{kumar2020conservative} and Implicit Q-Learning (IQL) \citep{kostrikov2021offline}. Strictly, BC is an imitation learning algorithm. We categorize BC here just to emphasize that it is not a sequence modeling algorithm. 
    \item \textbf{Sequence modeling} includes Decision Transformer (DT) \citep{chen2021decision}, Online Decision Transformer (ODT) \citep{zheng2022online}, Elastic Decision Transformer (EDT) \citep{wu2023elastic} and Q-learning Decision Transformer (QDT) \citep{yamagata2023q}.
\end{itemize}

\paragraph{Results on Dense Rewards Datasets} The results of \textbf{Rein\textit{for}mer} on three datasets are presented in Table \ref{tab1}.
Evaluations of existing sequence modeling algorithms are notably insufficient, primarily focusing on \texttt{Gym} datasets while overlooking more challenging \texttt{Maze2d} and \texttt{Kitchen}.
All the sequence model are able to achieve performance comparable to RL algorithms on \texttt{Gym} dataset. 
Furthermore, \textbf{Rein\textit{for}mer} also demonstrates superior performance on \texttt{Maze2d} (\textcolor{red}{$+37.07\%$}) and \texttt{Kitchen}(\textcolor{red}{$+28.45\%$}) compared to the strongest baseline.
Especially on \texttt{Kitchen}, classical offline algorithms do not significantly outperform BC, indicating that reward function is not fully utilized.

\paragraph{Results on Sparse Rewards Dataset} The \texttt{Antmaze} dataset features sparse rewards, with $r=1$ when reaching the goal. 
Both \texttt{medium-diverse} and \texttt{medium-play} does not contains complete trajectories from the starting point to the goal, which necessitates the algorithm to stitch failed trajectories to accomplish the goal. 
Despite the claims of trajectory stitching ability, our reproduced results in Table \ref{tab2} indicate that EDT \citep{wu2023elastic} performs poorly.

\begin{table}[htbp]
  \centering
  \vspace{-12pt}
  \caption{The normalized best score on D4RL \texttt{Antmaze-v2} dataset. We report the mean and standard deviation of normalized score for five seeds. For each seed, the stats is calculated by 100 evaluation trajectories. The best result is \textbf{bold} and the \textcolor{blue}{blue} result means the best result among sequence modeling.}
  \resizebox{\linewidth}{!}{
    \begin{tabular}{r|rr|rrrr}
    \toprule
    \multicolumn{1}{c|}{\multirow{2}[4]{*}{\texttt{Antmaze-v2}}} & \multicolumn{2}{c|}{RL} & \multicolumn{4}{c}{Sequence Modeling} \\
\cmidrule{2-7}          & \multicolumn{1}{c}{BC} & \multicolumn{1}{c|}{IQL} & \multicolumn{1}{c}{DT} & \multicolumn{1}{c}{EDT} & \multicolumn{1}{c}{ODT} & \multicolumn{1}{c}{\textbf{Rein\textit{for}mer}} \\
    \midrule
    \multicolumn{1}{l|}{\texttt{umaze}} & 68.5  & 84.0  & 64.5  & 67.8  & 53.1  & \textcolor{blue}{\textbf{84.4$\pm$2.7}} \\
    \multicolumn{1}{l|}{\texttt{umaze-diverse}} & 64.8  & \textbf{79.5}  & 60.5  & 58.3  & 50.2  & \textcolor{blue}{65.8$\pm$4.1} \\
    \multicolumn{1}{l|}{\texttt{medium-play}} & 4.5   & \textbf{78.5}  & 0.8   & 0.0   & 0.0   & \textcolor{blue}{13.2$\pm$6.1} \\
    \multicolumn{1}{l|}{\texttt{medium-diverse}} & 4.8   & \textbf{83.5}  & 0.5   & 0.0   & 0.0   & \textcolor{blue}{10.6$\pm$6.9} \\
\cmidrule{1-7}    \textit{Total} & \textit{142.6} & \textit{325.5} & \textit{126.3} &\textit{126.1} & \textit{103.3} & \textcolor{blue}{\textit{174.0}} \\
    \bottomrule
    \end{tabular}}%
    \vspace{-8pt}
  \label{tab2}%
\end{table}%

One observation in Table \ref{tab2} is that the previous sequence models even underperform BC, which highlights the absence of trajectory stitching. 
\textbf{Rein\textit{for}mer} exhibits a significant improvement compared to other sequence models, especially on \texttt{medium-diverse} and \texttt{medium-play}. 
However, compared to RL algorithm IQL \citep{kostrikov2021offline}, the performance gap remains quite considerable.

\paragraph{Summary and Discussion} Thanks to the max-return sequence modeling paradigm, \textbf{Rein\textit{for}mer} has emerged as the current state-of-the-art in sequence modeling. 
On \texttt{Antmaze} dataset that particularly demands extreme trajectory stitching, offline RL retains a significant advantage. 
However, on other datasets, our method demonstrates a noticeable improvement.
Figure \ref{fig:compare} describes the improvement probability \citep{agarwal2021deep} of \textbf{Rein\textit{for}mer}.
\begin{figure}[H]
% \vspace{-10pt}
    \centering
    \includegraphics[width=0.85\linewidth]{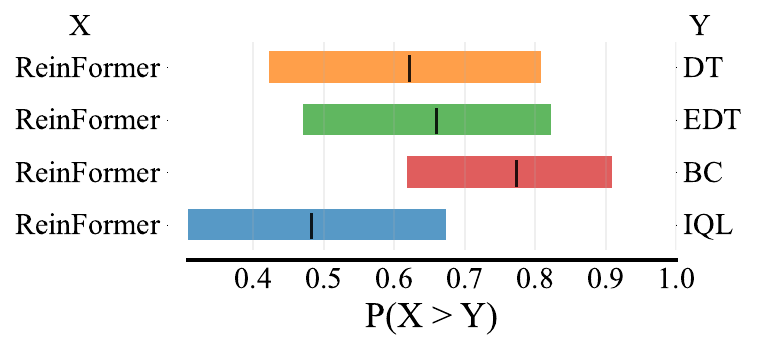}
    \vspace{-12pt}
    \caption{The probability of improvement of \textbf{Rein\textit{for}mer} compared with other methods using rliable \citep{agarwal2021deep}. The larger the probability is, the better our method performs.}
    \label{fig:compare}
\end{figure}

\subsection{Analysis of Returns-to-go loss}
We will delve into the critical return loss in max-return sequence modeling. 
We aim to understand the impact of this module on the \textbf{Rein\textit{for}mer} performance and explore the characteristics of the predicted maximized returns.
\vspace{-0.4cm}
\begin{figure}[H]
  \centering
  \hspace{-6pt}
  \subfigure{
    \includegraphics[width=0.49\linewidth]{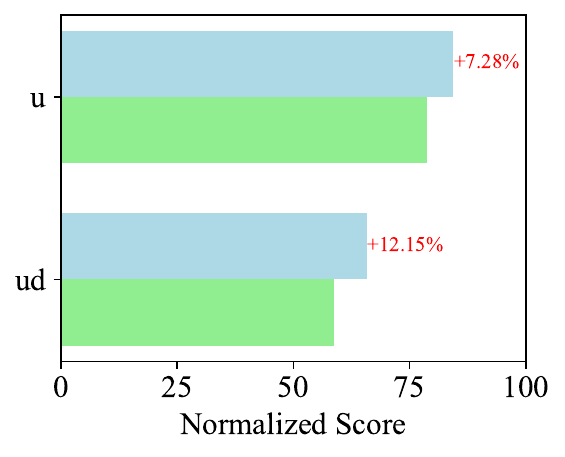}
    \label{fig:subfig1}
  }
  \hspace{-12pt}
  \subfigure{
    \includegraphics[width=0.49\linewidth]{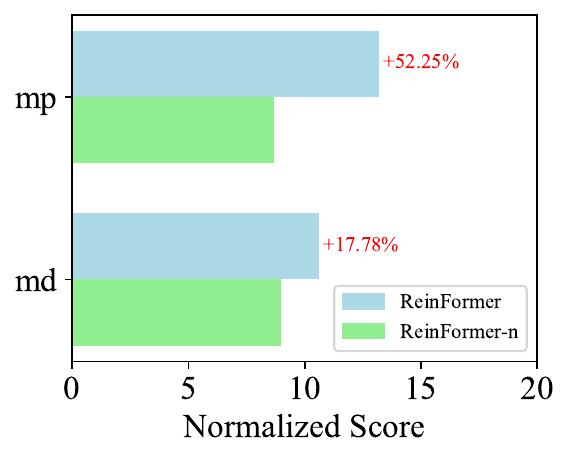}
    \label{fig:subfig2}
  }
  \vspace{-16pt}
  \caption{The comparison between different inference approaches on \texttt{Antmaze} dataset: \textbf{Rein\textit{for}mer} with predicted maximized returns versus \textbf{Rein\textit{for}mer}-n with naively maximized returns. Abbreviations: u $\rightarrow$ \texttt{umaze}, ud $\rightarrow$ \texttt{umaze-diverse}, mp $\rightarrow$ \texttt{medium-play}, md $\rightarrow$ \texttt{medium-diverse}.}
  \label{fig5}
\end{figure}
\vspace{-0.4cm}
\paragraph{Comparison on Maximizing Return Approach} The key advantage of max-return sequence modeling lies in predicting the maximized in-distribution return during the inference. 
The compared baseline \textbf{Rein\textit{for}mer}-n is based on the naive max approach that the initial returns is the max return in offline dataset.
This baseline also includes the returns-to-go loss during training. 
As for the comparison results without return loss, please refer to the results in Table \ref{tab1} or \ref{tab2}. 
We demonstrate the performance on different \texttt{Antmaze} datasets in Figure \ref{fig5}. 
The improvement observed on \texttt{umaze} and \texttt{umaze-diverse}, which contains the entire trajectory from the starting point to the goal, is \textcolor{red}{$19.43\%$}. 
The improvement on \texttt{medium-play} and \texttt{medium-diverse}, where emphasizes trajectory stitching, is \textcolor{red}{$70.03\%$}.

\paragraph{Ablation Study on $m$} Next, we analyze the impact of the hyper-parameter $m$ in return loss. 
According to the theorem \ref{thm}, when $m \rightarrow 1$, the learned returns will approach to the maximum return within the offline distribution.
Furthermore, due to the fact that higher in-distribution return leads to better action, we can conclude that the performance will improve as $m$ approaches $1$.
The experimental results in Figure \ref{fig6} indeed align with the above theoretical analysis.
Within a certain range, large $m$ generally leads to better training process and higher performance. 
However, when $m=0.999$ is excessively large, it may result in a performance decline. 
This could be attributed to the model overfitting to some extreme large returns in offline dataset.
\vspace{-0.4cm}
\begin{figure}[H]
  \centering
  \hspace{-6pt}
  \subfigure[\texttt{maze2d-umaze}]{
    \includegraphics[width=0.49\linewidth]{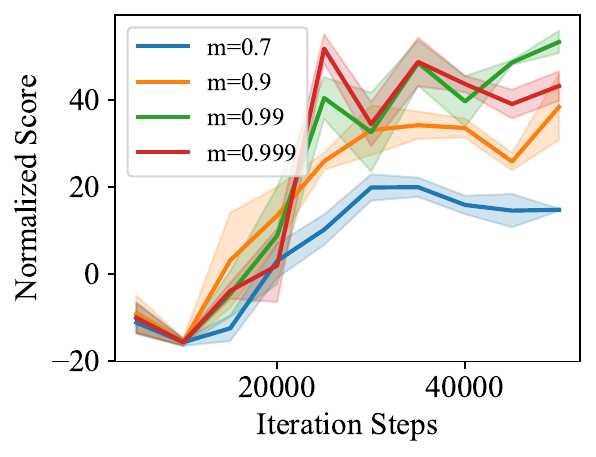}
    \label{fig:subfig3}
  }
  \hspace{-12pt}
  \subfigure[\texttt{Antmaze}]{
    \includegraphics[width=0.49\linewidth]{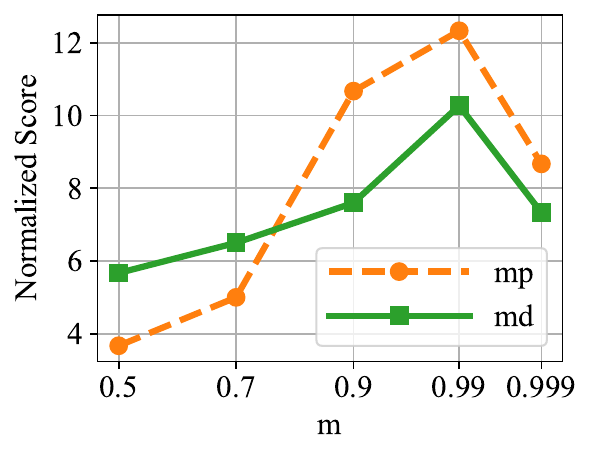}
    \label{fig:subfig4}
  }
  \vspace{-12pt}
  \caption{The performance given different hyper-parameter $m$. \ref{fig:subfig1} illustrates the training curves with varying $m$ on \texttt{maze2d-umaze} while \ref{fig:subfig2} illustrates the trend of best results as $m$ varies on \texttt{Antmaze-medium-play} (mp) and \texttt{medium-devise} (md).}
  \label{fig6}
\end{figure}
\vspace{-0.4cm}

\paragraph{Characteristic of the Predicted Maximized Returns}
We conduct experiments on \texttt{Halfcheetah-medium} and \texttt{Kitchen-complete} to analyze the characteristics of predicted maximized returns-to-go. 
In Figure \ref{fig:subfig6}, the dashed line represents the true returns-to-go. 
Original \texttt{Kitchen} environment is a staged sparse reward environment so the dashed line exhibits a step-wise decline. 
Trajectory return can only be one of $[0,1,2,3,4]$ but the dataset in D4RL \citep{fu2020d4rl} contains trajectories with returns ranging from $0$ to $400$. 
Therefore the predicted return, the solid line, is a continuous curve.
The may results from some densification operation when constructing the dataset.

In environment with non-negative rewards, the true returns-to-go should exhibit a monotonically decreasing trend and eventually reach 0.
In \ref{fig:subfig5}, the predicted returns with different $m$ all exhibits clear decreasing trend and only the excessively large $m=0.999$ is unable to converge to 0.
For \texttt{Kitchen} in Figure \ref{fig:subfig6}, the solid curve $g=4$ represents the predicted return of the trajectory that completes the task while $g=2$ is the trajectory that completes the half.
The curve $g=4$ demonstrates a good monotonically decreasing property.
But the descent rate of $g=2$ significantly slowed down in the latter part and ultimately did not converge to zero.
This high level return indicates that the model is still attempting to complete the task, although it ultimately fail.
\vspace{-0.8cm}
\begin{figure}[H]
  \centering
  \hspace{-6pt}
  \subfigure[\texttt{Halfcheetah-medium}]{
    \includegraphics[width=0.59\linewidth]{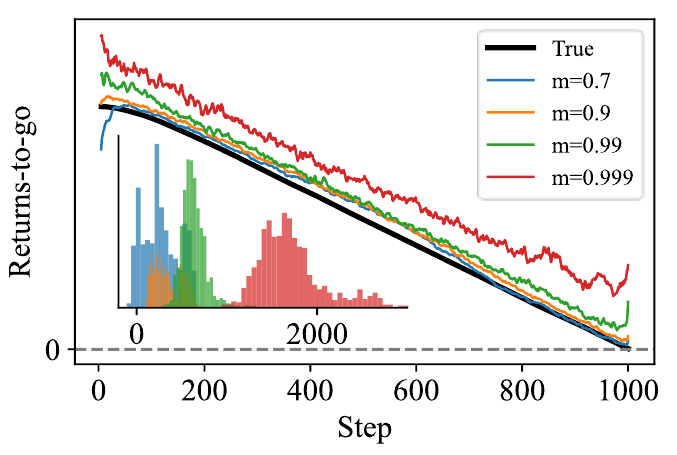}
    \label{fig:subfig5}
  }
  \hspace{-14pt}
  \subfigure[\texttt{Kitchen-complete}]{
    \includegraphics[width=0.40\linewidth]{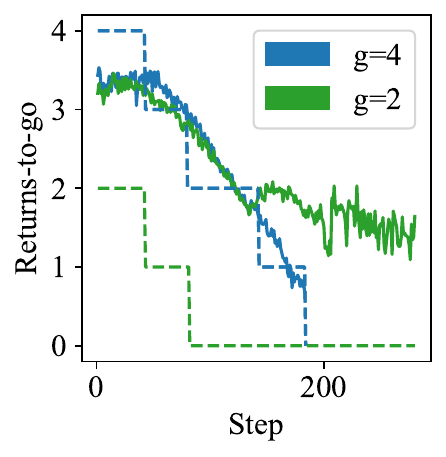}
    \label{fig:subfig6}
  }
  \vspace{-12pt}
  \caption{The predicted maximized returns-to-go by \textbf{Rein\textit{for}mer} and the true returns-to-go obtained through interaction with the environment. In \ref{fig:subfig3}, the black curve is the mean of true returns-to-go of different $m$ and the other curves are the predicted maximized returns-to-go. The bottom-left histogram illustrates the distribution of the differences between prediction and grand truth. In \ref{fig:subfig4}, the dashed line represents the true returns-to-go from the environment, while the solid line represents the predicted returns-to-go. It's worth noting that the trajectory return of \texttt{Kitchen} environment can only take one of $[0,1,2,3,4]$. Here, we present a trajectory that successfully completes the task ($g=4$) along with another one that completes half of the task ($g=2$).}
  \label{fig7}
\end{figure}
\vspace{-0.4cm}

From Figure \ref{fig:subfig5}, we observe that except for $m=0.999$, the predicted returns can well match the actual returns.
We plotted a frequency distribution histogram of the difference between the predicted return and the actual return. 
An observation is $m=0.999$ suffers from overestimation although this does not harm the performance.
In Figure \ref{fig:subfig6}, the case $g=4$ exhibits a relatively high degree of matching. 
Because case $g=2$ only completes the initial phase of the task, the predicted returns remains consistently high without any decrease or convergence. 
After translation, the true return still match the predicted maximized returns.

\section{Conclusion, Discussion and Future work}
In this work, we propose the paradigm of max-return sequence modeling which considers the RL objective that maximizes returns. 
Both theoretical analysis and experiments indicate the effectiveness of our proposed model \textbf{Rein\textit{for}mer}. 
Despite our promising improvement on trajectory stitching, sequence modeling still falls short compared to classical RL. 
In future work, we will focus on identifying and bridging the gap between classical RL algorithms and sequence modeling.
Furthermore, we aim to investigate the scenarios where classical RL excels and where sequence models can truly shine, providing a more nuanced understanding of their respective strengths and applications.

\section*{Acknowledgement}
This work was supported by the National Science and Technology Innovation 2030 - Major Project (Grant No. 2022ZD0208800), and NSFC General Program (Grant No. 62176215). 

\section*{Impact Statement}
Our research introduces a novel approach to offline reinforcement learning (RL) through the integration of max-return sequence modeling, which we term \textbf{Rein\textit{for}mer}. This work aims to address a critical gap in the current paradigm of sequence modeling by explicitly incorporating the core RL objective of maximizing returns. By doing so, we enhance the trajectory stitching capability, which is crucial for learning from sub-optimal data and improving the overall performance of RL algorithms.

The broader impact of our work extends to various domains where RL is applied, such as robotics, autonomous systems, and decision-making processes. The \textbf{Rein\textit{for}mer} algorithm the potential to lead to more efficient and effective learning algorithms, which could accelerate the development and deployment of autonomous technologies. This, in turn, could have significant societal consequences, including advancements in healthcare, transportation, and environmental management, where the ability to make optimal decisions is paramount.

From an ethical standpoint, our work emphasizes the importance of aligning RL algorithms with their intended objectives, ensuring that the pursuit of maximizing returns does not lead to unintended consequences. This is particularly relevant as AI systems become more integrated into critical systems where ethical considerations are paramount.

As we move forward, it is imperative to consider the societal implications of deploying RL algorithms that are more capable of learning from data. This includes ensuring that these systems are transparent, fair, and accountable, and that they do not perpetuate biases or lead to negative outcomes for vulnerable populations.

In conclusion, while our work contributes to the advancement of the field of Machine Learning, particularly in the area of RL, we recognize the need for ongoing discussions around the ethical use and societal impact of these technologies. We encourage the community to engage in these conversations and to develop guidelines that will ensure the responsible application of our findings.
\bibliography{example_paper}
\bibliographystyle{icml2024}

%%%%%%%%%%%%%%%%%%%%%%%%%%%%%%%%%%%%%%%%%%%%%%%%%%%%%%%%%%%%%%%%%%%%%%%%%%%%%%%
%%%%%%%%%%%%%%%%%%%%%%%%%%%%%%%%%%%%%%%%%%%%%%%%%%%%%%%%%%%%%%%%%%%%%%%%%%%%%%%
% APPENDIX
%%%%%%%%%%%%%%%%%%%%%%%%%%%%%%%%%%%%%%%%%%%%%%%%%%%%%%%%%%%%%%%%%%%%%%%%%%%%%%%
%%%%%%%%%%%%%%%%%%%%%%%%%%%%%%%%%%%%%%%%%%%%%%%%%%%%%%%%%%%%%%%%%%%%%%%%%%%%%%%
\newpage
\appendix
\onecolumn
\section{Baseline Results}\label{a1}
% Our baseline experiment results are from the original papers or reproduced by ourselves using official public code. 
% Now we clarify the data source.
Existing sequence modeling algorithms have only been evaluated on the \texttt{Gym-v2} dataset, neglecting other more complex and challenging datasets (\texttt{Maze2d-v1}, \texttt{Kitchen-v0} and \texttt{Antmaze-v2}). 
To provide a more comprehensive comparison and evaluation of \textbf{Rein\textit{for}mer}, we have reproduced several baseline methods using their offical code. 
Subsequently, we will clarify which results stem from the original works or third-party reproductions, and which are reproduced by ourselves.
Table \ref{tab1} contains \textbf{last} results on dataset \texttt{Gym-v2}, \texttt{Maze2d-v1} and \texttt{Kitchen-v0}.
\begin{itemize}
    \item \texttt{Gym-v2}: All the results of Reinforcement Learning baselines including BC, CQL and IQL come from the IQL original paper \citep{kostrikov2021offline}. 
    The results of Decision Transformer (DT) are from the DT original paper \citep{chen2021decision}. For ODT, EDT and QDT, the results are sourced from their respective original papers \citep{zheng2022online, wu2023elastic, yamagata2023q} while the results on dataset \texttt{medium-expert} are all absent. We reproduce the EDT which considers trajectory stitching using its official code \url{https://github.com/kristery/Elastic-DT}.
    \item \texttt{Maze2d-v1}: The results of BC, CQL, IQL and DT are from the open source library Clean Offline Reinforcement Learning (CORL) \url{https://github.com/tinkoff-ai/CORL}. CORL provides high-quality and easy-to-follow single-file implementations of state-of-the-art offline reinforcement learning algorithms and each implementation is backed by a research-friendly codebase \citep{tarasov2022corl}. For QDT, the results come from its original paper \citep{yamagata2023q} without discounted returns-to-go introduced. The result of EDT is still reproduced.
    \item \texttt{Kitchen-v0}: The results of BC, CQL and IQL come from the IQL original paper \citep{kostrikov2021offline} and we reproduce the EDT.
\end{itemize}
Table \ref{tab2} contains the \textbf{best} results on dataset \texttt{Antmaze-v2}. The reason for selecting the best result rather than the last one is due to the instability in the training of \textbf{Rein\textit{for}mer}. Additionally, our algorithm's performance on \texttt{Antmaze-large-diverse} and \texttt{Antmaze-large-play} is poor, with the agent completing the tasks only 1 or 2 times out of 100 evaluations. It is challenging to ascertain whether these results are attributable to the randomness induced by hyperparameter tuning or seed selection.
So we only report the results on other four datasets.
\begin{itemize}
    \item \texttt{Antmaze-v2}: The results of BC, IQL and DT are extracted from CORL library \cite{tarasov2022corl}. The other results are reproduced by ourselves using its their respective official code, including ODT \url{https://github.com/facebookresearch/online-dt} and EDT \url{https://github.com/kristery/Elastic-DT}.
\end{itemize}
During the process of reproducing EDT, we observe that EDT is highly sensitive to the initial returns-to-go. 
This indicates that, despite incorporating a loss function aimed at maximizing returns, EDT still operates with naively maximizing returns.
Additionally, we are not fully acquainted with the optimal selection of initial returns-to-go on \texttt{Maze2d-v1}, \texttt{Kitchen-v0} and \texttt{Antmaze-v2}. 
To address this, we have uniformly adopted the maximum return observed in the dataset. 
However, we are uncertain whether this choice may introduce any adverse effects on the experimental outcomes of EDT.
This indirectly highlights one of the advantages of \textbf{Rein\textit{for}mer}, which is its ability to adaptively select the returns-to-go without the need for manual design.

\section{Experiment Details and Hyperparameters}
Our \textbf{Rein\textit{for}mer} implementation draws inspiration from and references the following four repositories:
\begin{itemize}
    \item online-dt: \url{https://github.com/facebookresearch/online-dt};
    \item Elatic-DT: \url{https://github.com/kristery/Elastic-DT};
    \item min-decision-transformer: \url{https://github.com/nikhilbarhate99/min-decision-transformer}; 
    \item decision-transformer: \url{https://github.com/kzl/decision-transformer}.
\end{itemize}
The state tokens, return tokens and action tokens are first processed by different linear layers. 
Then these tokens are fed into the decoder layer to obtain the embedding.
Here the decoder layer is a lightweight implementation from ``min-decision-transformer''.
The context length for the decoder layer is denoted as $K$. 
We use the LAMB \citep{you2019large} optimizer with to optimize the model with action loss and returns-to-go loss. The hyperparameter of returns loss is denoted as $m$. 
\subsection{Hyperparameter $m$}
The hyperparameter $m$ is crucially related to the returns-to-go loss and is one of our primary focuses for tuning. 
We explore values within the range of $m = [0.7, 0.9, 0.99, 0.999]$. 
When $m=0.5$, the expectile loss function will degenerate into MSE loss, which means the model is unable to output a maximized returns-to-go.
So we do not take $m=0.5$ into consideration.
We observe that performance is generally lower at $m=0.7$ compared to others.
Only \texttt{hopper-medium} and \texttt{hopper-medium-replay} adopt the parameter $m=0.999$ while $m=0.9$ and $m=0.99$ are generally better than $m=0.999$ on other datasets. The detailed hyperparameter selection of $m$ is summarized in the following table:

\begin{table}[h]
\centering
    \caption{Hyperparameters $m$ of returns loss on different datasets. }
    \vspace{6pt}
    \begin{tabular}{l|l||l|l}
    \toprule
    \textbf{Dataset}                   & \textbf{$m$} & \texttt{maze2d-umaze}           & 0.99 \\ \cline{1-2}
    \texttt{halfcheetah-medium}        & 0.9          & \texttt{maze2d-medium}          & 0.99 \\
    \texttt{halfcheetah-medium-replay} & 0.9          & \texttt{maze2d-large}           & 0.99 \\ \cline{3-4}
    \texttt{halfcheetah-medium-expert} & 0.9          & \texttt{kitchen-complete}       & 0.99 \\
    \texttt{hopper-medium}             & 0.999        & \texttt{kitchen-partial}        & 0.9  \\
    \texttt{hopper-medium-replay}      & 0.999        & \texttt{kitchen-mixed}          & 0.9  \\ \cline{3-4}
    \texttt{hopper-medium-expert}      & 0.9          & \texttt{Antmaze-umaze}          & 0.9  \\
    \texttt{walker2d-medium}           & 0.9          & \texttt{Antmaze-umaze-diverse}  & 0.99 \\
    \texttt{walker2d-medium-replay}    & 0.99         & \texttt{Antmaze-medium-play}    & 0.99 \\
    \texttt{walker2d-medium-expert}    & 0.99         & \texttt{Antmaze-medium-diverse} & 0.99 \\
    \bottomrule
    \end{tabular}
\end{table}

\subsection{Context Length $K$}

The context length $K$ is another key hyperparameter for sequence modeling, and we conduct a parameter search across the values $K=[2, 5, 10, 20]$. 
The maximum value is $20$ because the default context length for DT \citep{chen2021decision} is $20$. 
The minimum is $2$, which corresponds to the shortest sequence length (setting $K=1$ would no longer constitute sequence modeling). 
Overall, we found that $K=20$ leads to more stable learning and better performance on high quality dataset (such as \texttt{Gym-medium-expert} and \texttt{Kitchen}). 
Conversely, a smaller context length is preferable on low quality dataset (such as \texttt{Gym-medium/medium-replay} and \texttt{Antmaze}).
The parameter $K$ has been summarized as follows:

\begin{table}[h]
\centering    
    \caption{Context length $K$ on different datasets. }
    \vspace{6pt}
    \begin{tabular}{l|r||l|r}
    \toprule
    \textbf{Dataset}                   & \textbf{$K$} & \texttt{maze2d-umaze}           & 20 \\ \cline{1-2}
    \texttt{halfcheetah-medium}        & 5          & \texttt{maze2d-medium}          & 10 \\
    \texttt{halfcheetah-medium-replay} & 5          & \texttt{maze2d-large}           & 10 \\ \cline{3-4}
    \texttt{halfcheetah-medium-expert} & 20          & \texttt{kitchen-complete}       & 20 \\
    \texttt{hopper-medium}             & 5        & \texttt{kitchen-partial}        & 20  \\
    \texttt{hopper-medium-replay}      & 5        & \texttt{kitchen-mixed}          & 20  \\ \cline{3-4}
    \texttt{hopper-medium-expert}      & 20          & \texttt{Antmaze-umaze}          & 2  \\
    \texttt{walker2d-medium}           & 5          & \texttt{Antmaze-umaze-diverse}  & 2 \\
    \texttt{walker2d-medium-replay}    & 2         & \texttt{Antmaze-medium-play}    & 3 \\
    \texttt{walker2d-medium-expert}    & 20         & \texttt{Antmaze-medium-diverse} & 2 \\
    \bottomrule
    \end{tabular}

\end{table}

\subsection{Training Steps and Learning Rate}

The default number of training steps is 50000, with a learning rate of 0.0001.
With these default settings, if the training score continues to rise, we would consider increasing the number of training steps or doubling the learning rate.
For some datasets, 50000 steps may cause overfitting and less training steps are better.
The training steps are presented in Table \ref{training_step} and learning rate is summarized in Table \ref{lr}.
We evaluate the policy every 5000 steps to obtain a normalized score.
For each seed, this normalized score is calculated as the average returns of 100 trajectories except for 10 trajectories on \texttt{Gym-v2} datasets.

\begin{table}[]
    \centering
    \caption{The training steps on different datasets.}
    \vspace{6pt}
    \begin{tabular}{l|r||l|r}
    \toprule
    \textbf{Dataset}                  & \textbf{Training Steps}          & \texttt{maze2d-umaze}           & 50000  \\ \cline{1-2}
    \texttt{halfcheetah-medium}        & 50000                   & \texttt{maze2d-medium}          & 35000  \\
    \texttt{halfcheetah-medium-replay} & 50000                   & \texttt{maze2d-large}           & 45000  \\ \cline{3-4}
    \texttt{halfcheetah-medium-expert} & 50000                   & \texttt{kitchen-complete}       & 100000 \\
    \texttt{hopper-medium}             & 30000                   & \texttt{kitchen-partial}        & 90000  \\
    \texttt{hopper-medium-replay}      & 300000                  & \texttt{kitchen-mixed}          & 50000  \\ \cline{3-4}
    \texttt{hopper-medium-expert}      & 100000                  & \texttt{Antmaze-umaze}          & 50000  \\
    \texttt{walker2d-medium}           & 15000                   & \texttt{Antmaze-umaze-diverse}  & 50000  \\
    \texttt{walker2d-medium-replay}    & 80000                   & \texttt{Antmaze-medium-play}    & 100000 \\
    \texttt{walker2d-medium-expert}    & 50000                   & \texttt{Antmaze-medium-diverse} & 100000 \\
    \bottomrule
    \end{tabular}
    
    \label{training_step}
\end{table}

\begin{table}[]
\centering
\caption{The learning rate on different datasets.}
\vspace{6pt}
    \begin{tabular}{r|l}
    \toprule
    \textbf{Learning Rate} & \textbf{Dataset}                                                     \\ \hline
    0.0008                            & \texttt{Antmaze-medium-play};                                 \\
    \multirow{3}{*}{0.0004}           & 1) \texttt{walker2d-medium-replay};                               \\
                                      & 2) \texttt{maze2d-medium}; 3) \texttt{maze2d-large};                 \\
                                      & 4) \texttt{Antmaze-medium-diverse};                              \\
    0.0001                            & Other datasets.     \\                            
    \bottomrule
    \end{tabular}
    \label{lr}
\end{table}

\subsection{Network Architecture}
The default values for the number of decoder layer, attention heads and hidden dimension are 4, 8 and 256, respectively. 
These parameters are usually fixed. 
When we observe an initial increase followed by a decrease in the training curve, we infer overfitting and reduce the number of layers. 
On the contrary, if the training curve consistently rises without a clear convergence trend, we would attempt to increase the number of layers. 
As for the number of attention heads, only \texttt{Antmaze-medium-diverse} is 4.

\begin{table}[h]
\centering
\caption{The number of encoder layers and attention heads on different datasets.}
\vspace{6pt}
\resizebox{\linewidth}{!}{
    \begin{tabular}{l|r|l}
    \toprule
    \textbf{Hyperparameters}                  & \textbf{Values} & \textbf{Dataset}                                                          \\
    \hline
    \multirow{4}{*}{Encoder Layers}  & 3                          & \texttt{maze2d-umaze}; \texttt{antmaze-umaze}; \texttt{antmaze-umaze-diverse};              \\
                                     & 4                          & \texttt{Gym-v2};                                                          \\
                                     & 5                          & \texttt{maze2d-medium}; \texttt{maze2d-large}; \texttt{Kitchen-v0}; \texttt{Antmaze-medium-diverse}; \\
                                     & 6                          & \texttt{Antmaze-medium-play}.                                             \\
    \hline
    \multirow{2}{*}{Attention Heads} & 4                          & \texttt{Antmaze-medium-diverse};                                          \\
                                     & 8                          & Other datasets.                                                   \\
    \bottomrule
    \end{tabular}}
    
\end{table}

\section{Training Curves}
We exhibit the training curves on five seeds.
The black line represents the mean of these five seeds and the red shaded area represents the variance.

\subsection{\texttt{Gym-v2}}
The training curves on nine datasets from \texttt{Gym-v2} are plotted in Figure \ref{gym}.
Among these nine datasets, \texttt{hopper-medium-replay} exhibits severe training fluctuations, while \texttt{walker2d-medium-replay} shows slight fluctuation. 
The remaining datasets are notably stable, yielding satisfactory results without deliberate hyperparameter tuning.

\begin{figure}[h]
\begin{center}
    \includegraphics[width=0.3\textwidth]{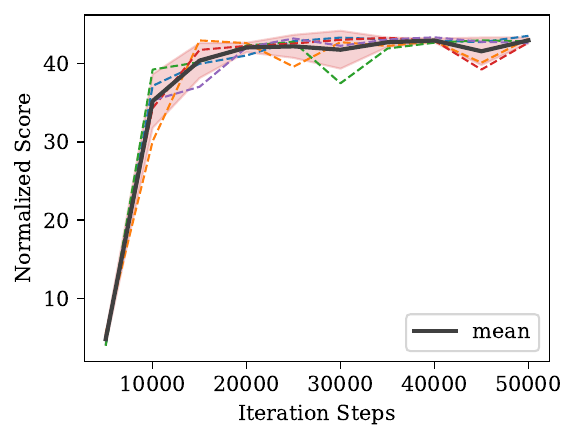}\label{fig:sub1}\hspace{0.3cm}
    \includegraphics[width=0.3\textwidth]{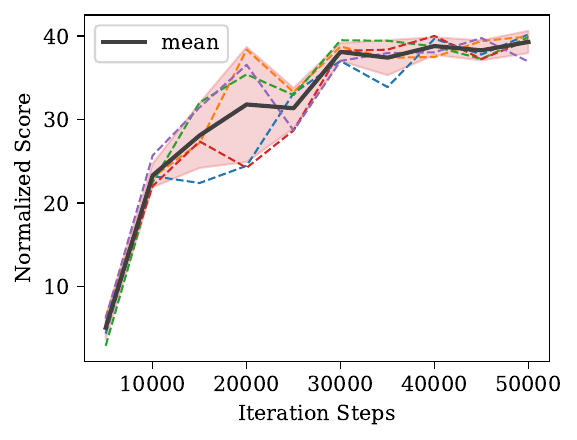}\label{fig:sub2}\hspace{0.3cm} 
    \includegraphics[width=0.3\textwidth]{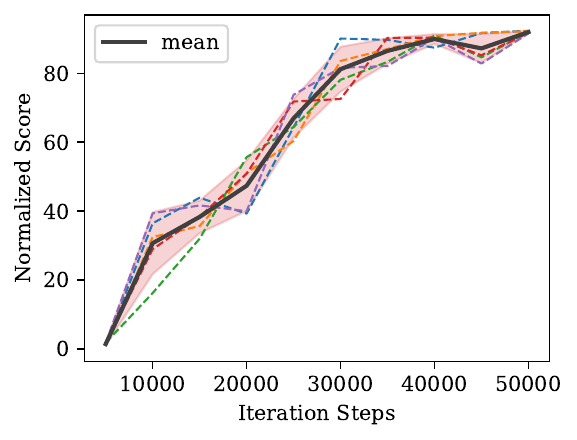}\label{fig:sub3}\\
\end{center}
\vspace{-0.3cm}
    \text{~~~~~~~~~~~~~(a)~\texttt{halfcheetah-medium}~~~~~(b)~\texttt{halfcheetah-medium-replay}~~~~(c)~\texttt{halfcheetah-medium-expert}}
\begin{center}
    \includegraphics[width=0.3\textwidth]{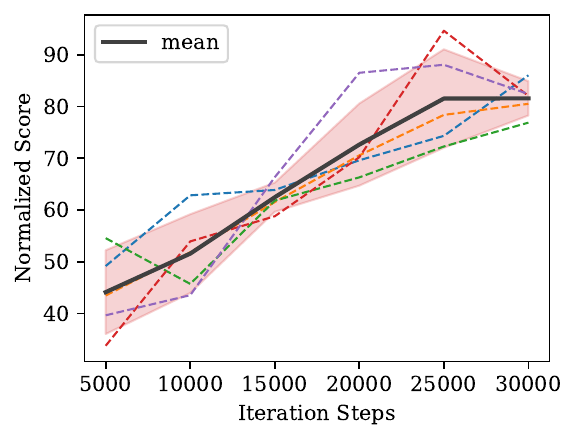}\hspace{0.3cm}\label{fig:sub4}
    \includegraphics[width=0.3\textwidth]{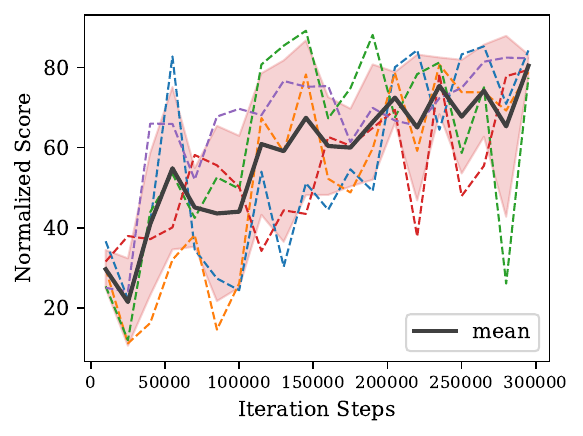}\hspace{0.3cm}\label{fig:sub5}
    \includegraphics[width=0.3\textwidth]{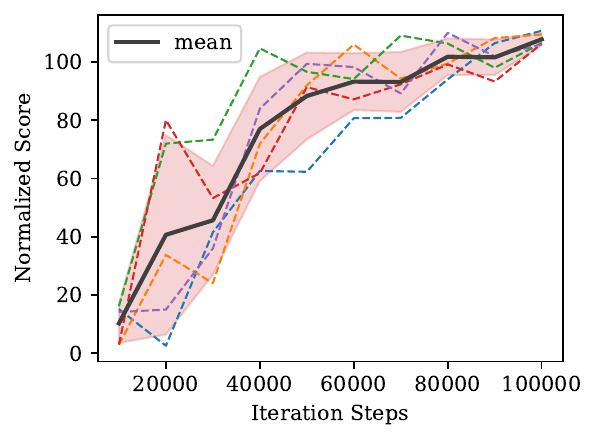}\label{fig:sub6}\\
\end{center}
\vspace{-0.3cm}
    \text{~~~~~~~~~~~~~~~~~~(d)~\texttt{hopper-medium}~~~~~~~~~~~~~~~~~~(e)~\texttt{hopper-medium-replay}~~~~~~~~~~~(f)~\texttt{hopper-medium-expert}}
\begin{center}
    \includegraphics[width=0.3\textwidth]{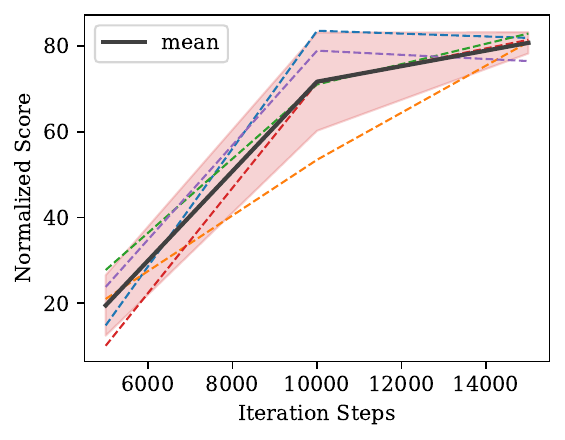}\hspace{0.3cm}\label{fig:sub7}
    \includegraphics[width=0.3\textwidth]{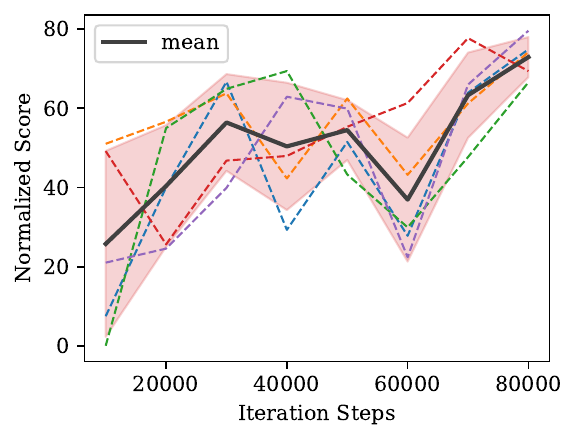}\hspace{0.3cm}\label{fig:sub8}
    \includegraphics[width=0.3\textwidth]{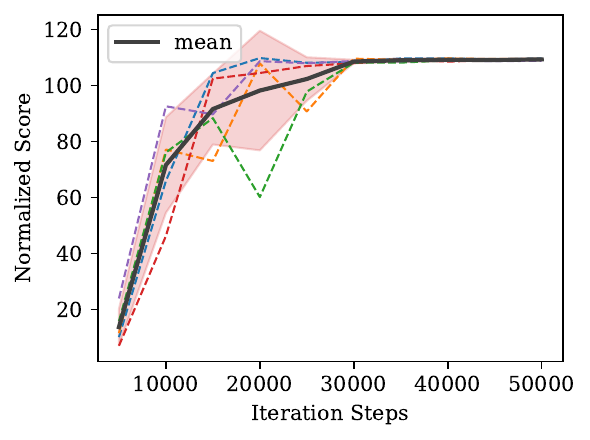}\label{fig:sub9}\\
\vspace{-0.1cm}
    \text{~~~~~~~~~~~~~~~~~(g)~\texttt{walker2d-medium}~~~~~~~~~~~~(h)~\texttt{walker2d-medium-replay}~~~~~~(i)~\texttt{walker2d-medium-expert}}
\end{center}
  \caption{The training curves on \texttt{Gym-v2}.}
  \vspace{-6pt}
  \label{gym}
\end{figure}

\subsection{\texttt{Maze2d-v1}}
The training curves on three datasets from \texttt{Maze2d-v1} are plotted in Figure \ref{maze2d}.
The training curve of \texttt{maze2d-umaze} is relatively stable. 
The variance on dataset \texttt{maze2d-medium} is very high and the training process also suffers from severe training fluctuations.
Sometimes, the score can even approach to 125.
Dataset \texttt{maze2d-medium} also fluctuates a little.

\begin{figure}[!b]
\vspace{-6pt}
\begin{center}
    \includegraphics[width=0.3\textwidth]{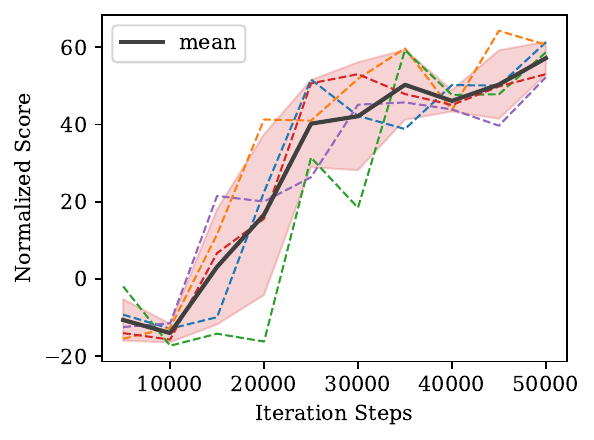}\label{fig:sub10}\hspace{0.3cm}
    \includegraphics[width=0.3\textwidth]{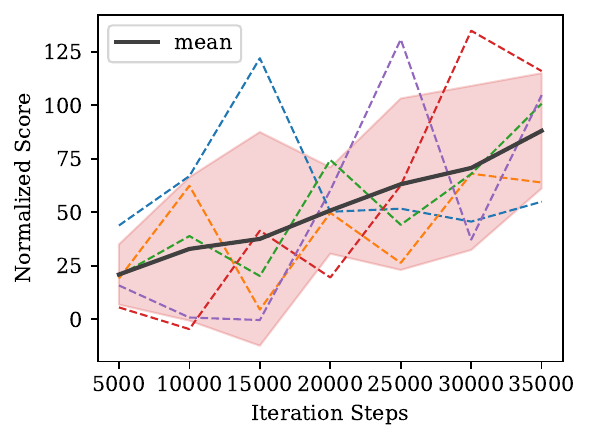}\label{fig:sub11}\hspace{0.3cm} 
    \includegraphics[width=0.3\textwidth]{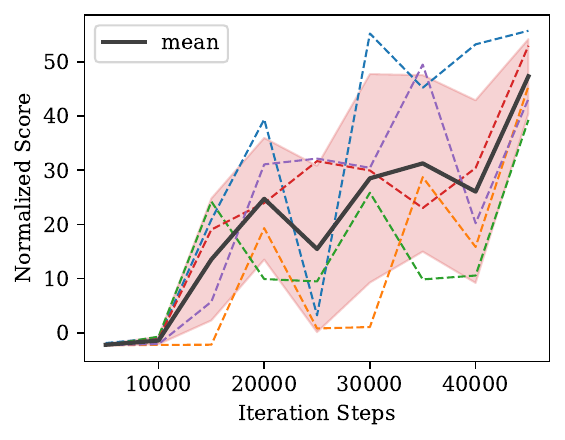}\label{fig:sub12}\\
\end{center}
\vspace{-0.3cm}
    \text{~~~~~~~~~~~~~~~~~~~~(a)~\texttt{maze2d-umaze}~~~~~~~~~~~~~~~~~~~~~~~~~~~~(b)~\texttt{maze2d-medium}~~~~~~~~~~~~~~~~~~~~~~~~~~~~~(c)~\texttt{maze2d-large}}
  \caption{The training curves on \texttt{Maze2d-v1}.}
  \vspace{-3pt}
  \label{maze2d}
\end{figure}

\subsection{\texttt{Kitchen-v0}}
The training curves on three datasets from \texttt{Kitchen-v0} are plotted in Figure \ref{maze2d}.
Overall, the training curves are remarkably stable, and the long context length $K$ is crucial for the stability of learning.

\begin{figure}[h]
\begin{center}
    \includegraphics[width=0.3\textwidth]{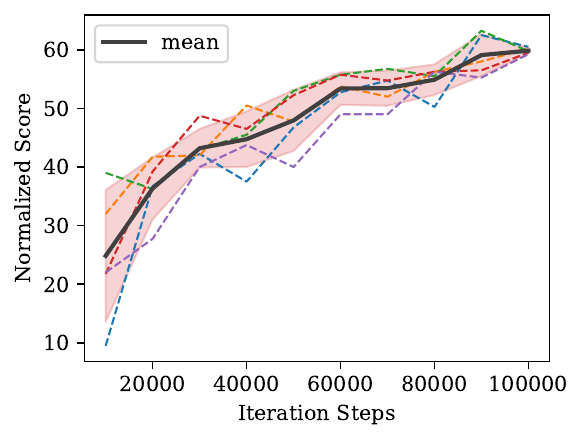}\label{fig:sub13}\hspace{0.3cm}
    \includegraphics[width=0.3\textwidth]{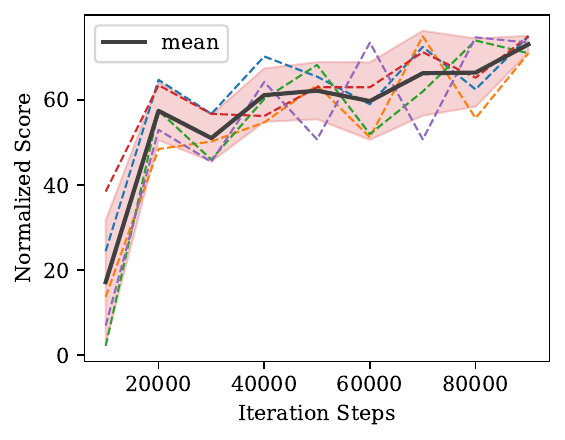}\label{fig:sub14}\hspace{0.3cm} 
    \includegraphics[width=0.3\textwidth]{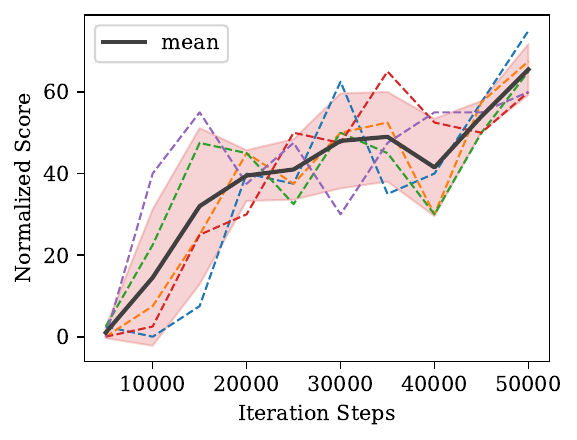}\label{fig:sub15}\\
\end{center}
\vspace{-0.3cm}
    \text{~~~~~~~~~~~~~~~(a)~\texttt{kitchen-complete}~~~~~~~~~~~~~~~~~~~~~(b)~\texttt{kitchen-partial}~~~~~~~~~~~~~~~~~~~~~~~(c)~\texttt{kitchen-mixed}}
  \caption{The training curves on \texttt{Kitchen-v0}.}
  \vspace{-6pt}
  \label{kichen}
\end{figure}

\subsection{\texttt{Antmaze-v2}}
Since we report the best score during training rather than the last score, we do not provide training curves on \texttt{Antmaze}.
Here, we would like to emphasize the reward modification.
\texttt{Antmaze} contains datasets with sparse rewards, where only 1 indicates the reach of the goal while 0 is not.
In order to enhance the reward signal, we multiply the reward by 100.
However, we found that this modification leads to the occurrence of NaN values on dataset \texttt{Antmaze-umaze-diverse}.
Besides, the original reward also occur the NaN values.
So we modify the reward by adding another 1, that is, $\hat{r}=100 \times r + 1$.
We summarize the results of different reward modification in Table \ref{reward}.
\begin{table}[h]
\centering
\caption{The normalized best score on \texttt{Antmaze-v2} dataset with different reward modification.}
\label{reward}
\vspace{6pt}
\begin{tabular}{l|rr}
\toprule
\texttt{Antmaze}        & $\hat{r}=100 \times r$ & $\hat{r}=100 \times r + 1$   \\
\hline
\texttt{umaze}          & \textbf{85$\pm$3.87}   & 84.4$\pm$2.7 \\
\texttt{umaze-diverse}  & NaN           & \textbf{65.8$\pm$4.1} \\
\texttt{medium-play}    & 11.4$\pm$3.78 & \textbf{13.2$\pm$6.1} \\
\texttt{medium-diverse} & 7.2$\pm$2.17  & \textbf{10.6$\pm$6.9} \\
\bottomrule
\end{tabular}
\end{table}

%%%%%%%%%%%%%%%%%%%%%%%%%%%%%%%%%%%%%%%%%%%%%%%%%%%%%%%%%%%%%%%%%%%%%%%%%%%%%%%
%%%%%%%%%%%%%%%%%%%%%%%%%%%%%%%%%%%%%%%%%%%%%%%%%%%%%%%%%%%%%%%%%%%%%%%%%%%%%%%

\end{document}